\documentclass[anonymous]{article}
\usepackage{ijcai22}
\usepackage{stmaryrd}
\usepackage{times}

\usepackage{soul}
\usepackage{url}
\usepackage[hidelinks]{hyperref}
\usepackage[utf8]{inputenc}
\usepackage[small]{caption}
\usepackage{graphicx}
\usepackage{amsmath}
\usepackage{booktabs}
\usepackage{xcolor}
\usepackage{pgfplots}
\usepackage[ruled,vlined,linesnumbered]{algorithm2e}
\usepackage{amsthm}
\usepackage{theoremref}
\usepackage{amsfonts}
\usepackage{amssymb}

\newtheorem{theorem}{Theorem}[section]
\newtheorem{proposition}[theorem]{Proposition}

\SetKwComment{Comment}{$\triangleright$\ }{}

\definecolor{ballblue}{rgb}{0.13, 0.67, 0.8}

\newcommand{\rr}{\mathbf{r}}
\newcommand{\pp}{\mathbf{p}}
\newcommand{\nn}{\mathbf{n}}
\newcommand{\xx}{\mathbf{x}}
\newcommand{\ww}{\mathbf{w}}

\newcommand{\cX}{\mathcal{X}}

\newcommand{\cP}{\mathcal{P}}
\newcommand{\cN}{\mathcal{N}}

\newcommand{\cS}{\mathcal{S}}
\newcommand{\cH}{\mathcal{H}}
\newcommand{\method}{FIND-RS}

\newcommand{\Real}{\mathbb{R}}
\newcommand{\Exp}{\mathbb{E}}

\title{Bayes Point Rule Set Learning}

\author{
Fabio Aiolli$^1$\and
Luca Bergamin$^1$\and
Tommaso Carraro$^{1,2}$\And
Mirko Polato$^3$\footnote{Contact Author}\\
\affiliations
$^1$Department of Mathematics, University of Padova, Padova, Italy\\
$^2$Fondazione Bruno Kessler, Trento, Italy\\
$^3$Department of Computer Science, University of Turin, Turin, Italy\\
\emails
aiolli@math.unipd.it,
luca.bergamin@math.unipd.it,
tcarraro@fbk.eu,
mirko.polato@unito.it
}

\begin{document}

\maketitle

\begin{abstract}
Interpretability is having an increasingly important role in the design of machine learning algorithms. However, interpretable methods tend to be less accurate than their black-box counterparts.
Among others, DNFs (Disjunctive Normal Forms) are arguably the most interpretable way to express a set of rules. In this paper, we propose an effective bottom-up extension of the popular FIND-S algorithm to learn DNF-type rulesets. The algorithm greedily finds a partition of the positive examples. The produced DNF is a set of conjunctive rules, each corresponding to the most specific rule consistent with a part of positive and all negative examples. We also propose two principled extensions of this method, approximating the Bayes Optimal Classifier by aggregating DNF decision rules. Finally, we provide a methodology to significantly improve the explainability of the learned rules while retaining their generalization capabilities.
An extensive comparison with state-of-the-art symbolic and statistical methods on several benchmark data sets shows that our proposal provides an excellent balance between explainability and accuracy.
\end{abstract}

\section{Introduction}

Despite being out of fashion, one of the most popular learning schemes in machine learning is rule learning. In the early years of machine learning, rule learning has dominated the research scene with work dated back to the 1960s \cite{Michalski1969}.
However, in the last decades, many statistical machine learning methods (e.g., SVM and neural networks) have shown of being superior in terms of predictive accuracy, absorbing most of the research community's efforts.
This ``pursuit of accuracy'' led to a period that we may define the winter of rule learning. 

The drawback of many state-of-the-art models is that they are (usually) complex and lack transparency, which leads to poor comprehensibility. Nowadays, more than ever, there is an increasing awareness of the importance of having the ability to explain the decisions of artificial intelligence systems.
One of the main lines of research that pursue the goal to better explain machine learning models is \textit{interpretable machine learning}, of which rule learning represents one of the go-to approaches. The main challenge in rule learning is to efficiently learn rules that are ``simple enough'' to be useful in practice but at the same time have a state-of-the-art accuracy. The literature offers many algorithms for rule induction~\cite{Prati2005ROCCERAA,ripper,cn2}, but still most of them are far from having performance comparable to the state-of-the-art.

Rule sets~\cite{Furnkranz2010}  are by far the most discussed approaches, thanks to their natural interpretation. In particular, for binary classification problems, the rules can be combined so that the resulting hypothesis resembles a Boolean formula in Disjunctive Normal Form (DNF): the set contains only rules that describe the positive class (concept learning~\cite{mitchellML}). Thus, an instance is classified as positive if and only it satisfies the conditions of at least one rule.

DNF-type rulesets are at the same time intuitive and powerful, i.e., a single DNF can describe any Boolean concept.
About DNF, Leslie Valiant, pioneer of computational learning theory, stated~\cite{Valiant_ijcai_1985}:

\textit{The possible importance of disjunctions of conjunctions as a knowledge representation stems from the observations that [...] humans appear to like using it.}

In this paper, we propose \method\ (Find Rule Set), an effective bottom-up extension of the popular FIND-S algorithm for DNF-ruleset induction. 
This method applies a specific-to-general approach by greedily partitioning the set of positive examples while it learns for each partition the most specific conjunctive rule consistent with the training set. Conjunctive rules are learned using the same idea of  FIND-S~\cite{mitchellML} and the final hypothesis consists of a disjunction of such conjunctive terms (i.e., a DNF). 
Under mild conditions, \method\ guarantees to find a hypothesis that correctly classifies the training set. We also provide an efficient pruning procedure that tries to simplify the final DNF to improve interpretability. 

As typical in the context of rule learning, on more complex classification tasks, the hypothesis of \method{} may perform poorly w.r.t. state-of-the-art classification methods. For this reason, we propose \method-BO, inspired by the Bayes Optimal Classifier~\cite{mitchellML,DBLP:conf/nips/HerbrichG00a}, which builds a classification hypothesis aggregating the decision of different \method's hypotheses on different runs.
Inspired by the work on Bayes Point Machines \cite{DBLP:conf/nips/HerbrichG00a,DBLP:journals/jmlr/HerbrichGC01}, we then define a principled method to approximate the center of mass of the version space. This method, named  \method-BP, has the advantage to produce an interpretable and very accurate set of rules.  
We show that this new hypothesis significantly improves the performance achieving state-of-the-art results. Finally, we provide a heuristic to  drastically simplify  the rule set generated, thus improving the explainability of \method-BP while retaining excellent generalization capabilities.


\section{Related Work}



Early work on rule learning focuses on employing different heuristic methods. Iterative Dichotomiser 3 (ID3, \cite{id3-1}, \cite{id3-2}), and later Classification and Regression Trees (CART, \cite{cart}), both use greedy partitioning based on different impurity measures. Another specialized learning algorithm class is sequential covering, which uses an iterative process to learn a new rule, useful to explain a subset of the training data. Once a rule is discovered, they remove the corresponding examples from the training process. A selection of the most relevant work is AQ \cite{aq}, CN2 \cite{cn2}, and RIPPER \cite{ripper}. The latter, based on IREP \cite{irep}, overcomes the overfitting problem with a more effective pruning phase while retaining computational efficiency. It is a fast method that achieves good performance to this day, hence still considered state-of-the-art \cite{overviewrulelearning}. The main weaknesses of sequential covering algorithms sit on falling short in processing efficiency for ever-growing modern data sets and on the inability to cover the whole search space, OPUS being the most notable exception (\cite{opus}).


A different class of algorithms employs Bayesian methods to evaluate proper rules in a probabilistic framework, more apt to be resilient to noise and uncertainty. Notable examples are Bayesian Rule Lists (BRL) \cite{brl} and the subsequent Scalable Bayesian Rule Lists (SBRL)  \cite{sbrl}. Albeit their foundation is sound, the required computation of frequent itemsets increases training time. Moreover, high memory is a common requirement. Hence, these methods do not scale well.

A common way to improve machine learning methods uses multiple instances of a simpler classifier. Both SLIPPER \cite{slipper} and Random Forest \cite{rf} are algorithms that leverage this idea. The former opts for RIPPER as the base learner, while the latter uses decision trees. Resorting to a voting strategy is a simple but effective way to boost performance while improving generalization. Albeit more effective, ensemble models are not easily interpretable by default, even if made up by interpretable models.  

Rule learning has seen a resurgence in recent years, using neural architectures based on gradient descent. Thanks to the notable endeavors of the deep learning research community, these methods are nowadays competitive to heuristic models \cite{rrl}, contributing to give new life to rule learning and unique perspectives on the nature of concept learning, e.g. DRNC \cite{beck2021empirical}. These promising methods have typical pitfalls of deep learning, such as sensitivity to hyper-parameters, high computing requirements, and convergence issues.

\section{Background}
In this section, we discuss the notions and notation that we will use throughout the paper.
\subsection{Notation}
We consider binary classification problems with training sets $\cS \equiv \{(\xx_i, y_i)\}_{i=1}^n$, where $\xx_i \in \cX$ are categorical feature vectors, with $\cX \equiv \times_{j=1}^m \cX_j$ for some finite (symbolic) attribute/variable domains $\cX_j$, and $y_i \in \{-1,+1\}$. We call $\cP \equiv \{\xx \mid (\xx, y) \in \cS \wedge y=+1 \}$ the set of positive instances, and conversely, $\cN \equiv \{\xx \mid (\xx, y) \in \cS \wedge y=-1 \}$ the set of negative instances.
We denote with $\rr \in \times_{i=1}^m (\cX_i \cup \{?\})$ a rule, and we say $\rr$ covers an instance $\xx \in \cX$ (or $\xx$ satisfies $\rr$), $\rr \succeq \xx$, iff $\forall i \in [m]$, $r_i =\: ?$ or $x_i = r_i$. In logical terms, $\rr \succeq \xx$ means that the conjunction $\bigwedge_{r_i \in \rr | r_i \neq ?} (r_i=x_i)$ is true. It is noteworthy that an instance $\xx$ can be seen as a rule in which every variable is constrained.
A set $D \equiv \{\rr_1, \dots, \rr_k\} $ of (conjunctive) rules covers an instance $\xx \in \cX$, denoted $D \succeq \xx$,  iff $\exists \rr \in D \mid \rr \succeq \xx$. From a logical stand point, a set of conjunctive rules corresponds to a monotone DNF (MDNF).
A DNF is called monotone if the terms in its conjunctions are all \textit{positive}, i.e., non-negated. We say that a DNF formula is consistent with a training set iff it covers the positive training instances and it does not cover any negative training instance. We say that a rule set (or DNF) $D_1$ generalizes a rule set $D_2$, $D_1 \geq D_2$,  iff $\forall \xx \in \cX$, $(D_2 \succeq \xx) \Rightarrow (D_1 \succeq \xx)$. With a slight abuse of notation, an MDNF rule $h$ is often used as a classification function in such a way that given an example $\xx$, $h(\xx) = +1$ if $h \succeq \xx$, $-1$ otherwise. An ordered set of rules is denoted by $\langle \rr_1, \dots, \rr_k \rangle$. Finally, $\llbracket b \rrbracket \in \{0,1\}$ denotes the indicator function which is 1 iff the condition $b$ is true.

\subsection{Hypothesis spaces}

Categorical variables, where no ordinal relationships exist in their domains, can be embedded in a vector by one-hot encoding (OH). Rule learners typically permit attribute-value encodings (AV) as well. Choosing one encoding or another can be crucial for the success of learning. In particular, the number and the nature of the possible rules changes. Let $k_i$ be the number of possible values for the attribute $a_i$, then the number of available rule-terms for the attribute $a_i$ is $N_{\text{AV}}(i) = k_i+1$ and $N_{\text{OH}}(i) = 2^{k_i}-1$ for the AV and OH encodings, respectively. For example, for an attribute with $3$ values \{1,2,3\}, we have $(1), (2), (3), (?)$ for the AV encoding and $(1,0,0), (0,1,0), (0,0,1), (0,?,?), (?,0,?), (?,?,0), (?,?,?)$ for the OH encoding, thus including negations, e.g., $(0,?,?)$ which is equivalent to $not\ (a_i=1)$. 
Hence, the number of different syntactic rules we can construct for an attribute is $N_{\text{AV}} = \prod_{a_i} (k_i + 1)$ and $N_{\text{OH}} = \prod_{a_i} (2^{k_i} - 1)$, respectively. Given $N$, the number of possible rules, we can have up to $N(N-1)\cdots(N-k+1)$ different k-terms MDNFs.

The methods proposed in this paper use the hypothesis space of MDNFs, that is $\cH = \{h | h\ \text{is an}\ \text{MDNF}\}$. Given a training set $\cS$, the \emph{version space} is the set of hypotheses that correctly classify training data, i.e., $\text{V}(\cS) = \{h \in \cH | h(\xx_i)=y_i, \forall (\xx_i,y_i) \in \cS\}$.

\section{\method}

In this section, we present \method\ a novel greedy concept learning method which extends the popular FIND-S algorithm. 

\subsection{\method\ overview}

\method\ considers a hypothesis space consisting of MDNF formulas of arbitrary size. \method\ uses a bottom-up approach, starting from a very specific hypothesis that is greedily generalized to allow novel positive instances to be covered, while being consistent with the negative instances.
\method\ is greedy because each positive example is seen only once\footnote{During the post-processing step a subset of the positive instances  are considered a second time.} and, at each iteration, the (running) hypothesis can only be generalized.

The algorithm updates the running MDNF hypothesis in a  fashion similar to the FIND-S algorithm. Specifically, the algorithm starts with an empty MDNF (i.e., the bottom rule), then, after considering a new positive training instance $\pp \in \cP$, it adapts the current MDNF in order to be satisfied by $\pp$. An update consists in generalizing a single conjunctive rule, while keeping track of the instances that have been used to modify that rule.

Let $D^t = \langle \rr_1, \dots, \rr_k \rangle$ be the hypothesis at iteration $t$ (i.e., after seeing $t$ positive instances). For each $\rr_i$, let $B_i$ (that we called bucket) be the set of instances used to update $\rr_i$ up to the current iteration. 
Given a new positive example $\pp$, the algorithm checks (following the order) if a rule $\rr_i$ can be generalized to cover $\pp$. If so, then the $\rr_i$ is updated (if necessary) and $\pp$ is added to $B_i$. The generalization procedure is carried out in such a way that the constructed (conjunctive) rule is equivalent to the FIND-S hypothesis built on its corresponding set of examples $B_i$. 
Akin FIND-S, \method\ generalizes a rule by removing one or more attribute-value constraints (i.e., setting it equals to $?$) while keeping the overall hypothesis consistent with the negative examples. 

Otherwise, if no $\rr_i$ can be safely generalized to cover $\pp$, then $D^{t+1} = D^{t} \vee \pp$, or in other words the generalization of the current hypothesis is done by adding a new rule that is satisfied by only instances that are equivalent to $\pp$.

It is worth to notice that, by design, \method\ will always find a (MDNF) hypothesis that correctly classifies all the training set iff there are no contradictory examples, that is $\nexists (\xx_1, y_1), (\xx_2, y_2) \in \cS \mid (\xx_1 = \xx_2) \wedge (y_1 \neq y_2)$. In the worst case, \method\ will return a hypothesis of the form $\bigvee_{\pp \in \cP} \pp$ that clearly overfits the training set.

\subsection{MDNF pruning} \label{sec:pruning}
Being \method\ a greedy algorithm, at the end of the training process, the produced hypothesis may contain superfluous rules, i.e., rules that can be removed without losing the consistency with the training set. Thus the idea of the pruning is to discard such redundant rules. To speed up this pruning process, we rely on the following observation.

\begin{proposition}\label{prop1}
At every iteration $t$ of \method\ holds that
$$
    \forall j < i \in [k], \; \nexists \xx \in B_i \mid \rr_j \succeq \xx,
$$
where the current hypothesis is $D^{t} = \langle \rr_1, \dots, \rr_k \rangle$.
\end{proposition}

\begin{proof}
We give a proof by induction on the number iterations $t$.
\textit{Case base}: $t=0$, then, by design, $D^0 = \emptyset$, thus Prop. \ref{prop1} is hollowly true;
\textit{Inductive step}: let us assume that Prop. \ref{prop1} holds up until iteration $t$ in which $D^t =\langle \rr_1, \dots, \rr_k \rangle$ where each conjunction $h_i$ has been learnt using FIND-S over $B_i$. At iteration $t+1$, a new positive instance $\xx$ is considered, and only one of the following two scenario must be true. $(1)$ $D^t \nsucceq \xx$, i.e., $\nexists \rr_i \mid$ generalize$(\rr_i, \xx)$ returns a hypothesis consistent with $\cN$, thus $D^{t+1} = D^t \vee \rr_{k+1}$ where $\rr_{k+1}=\xx$. Hence, by construction, $\nexists i \leq k \mid \rr_i \succeq \xx$, so Prop.\ref{prop1} holds.
$(2)$ $D^t \succeq \xx$, i.e., $\exists h_i \mid \rr'= $ generalize$(\rr_i, \xx) \succeq \xx$ and $\rr'$ is consistent with $\cN$. Since \method\ tries to generalize each conjunctive term in order, then $\forall j<i$, $\rr_j$ can not be safely generalized. After the generalization step, $\rr'$ can be either equals to $\rr_i$ (i.e., $\rr_i \succeq \xx$) or it can be a generalization of $\rr_i$, i.e., $\rr' \geq \rr_i$. In the former case, the overall hypothesis does not change and, for the just made considerations, Prop. \ref{prop1} holds. In the latter case, we have to show that $\rr'$ does not cover any instance in $B_j$ for $j>i$. Since $\rr_i$ has been created before any $\rr_{j>i}$ then $\rr_i$ can not be generalized (e.g., $\rr'$) such to cover any $\xx \in B_{j>i}$ otherwise it would have been previously generalized to cover $\xx$. Thus, Prop. \ref{prop1} holds.
\end{proof}

Proposition \ref{prop1} provides a "backward incompatibility" between rules, however it does not say anything in the other direction. In particular, as just mentioned, it may happen that every instance in $B_i$, for some $i$, can be covered by other conjunctive rules $\rr_j$ for $j>i$. 
In practice, this post-processing step (function \texttt{prune} in Algorithm~\ref{alg:finddnf}) checks if some $B_i$ can be emptied. In such a case, it means that the corresponding conjunctive term can be safely removed from the current hypothesis, thus creating a more specialized one that is still consistent with the training set.

The pseudo-code of \method\ is provided in Algorithm~\ref{alg:finddnf}.

\begin{algorithm}
    \caption{\method}
    \label{alg:finddnf}
	\DontPrintSemicolon
	\KwIn{
		$\cP$: set of positive examples;\\
		\hspace{3.35em}$\cN$: set of negative examples\\
	}
	\KwOut{
		$D$: Disjunctive Normal Form rule
	}
	\BlankLine
	$D, B, k \gets [], [], 0$ \; 
	\For {$\pp \in \cP$} {
	    done $\gets$ False \;
	    \Comment*[l]{\scriptsize{The conjunctive terms are taken in order}}
	    \For {$i \in [1, k]$} {
	        $\rr \gets D_i$ \;
	        $\rr' \gets$ generalize$(\rr, \pp)$ \;
	        \If {$\nexists \nn \in \cN \mid \rr' \succeq \nn$} {
	            $D_i \gets \rr'$ \;
	            $B_i \gets B_i \cup \{\pp\}$ \Comment*[r]{\scriptsize{Update the bucket}}
	            done $\gets$ True \;
	            break \;
	        }
	    }
	    \If {\normalfont{not done}} {
	        $k \gets k+1$ \;
	        $D_k \gets \pp$ \Comment*[r]{\scriptsize{Add a new conjunctive rule}}
	        $B_k \gets \{\pp\}$ \Comment*[r]{\scriptsize{Append a new bucket }}
	    }
    }
    $D \gets$ prune$(D)$ \;
    \Return $D$\;
\end{algorithm}

\subsection{Dealing with noisy data}

To cope with noisy datasets or harder concepts, we introduced in \method\ a tolerance hyper-parameter $\tau \in \mathbb{N}$. $\tau$ represents the number of negative instances of the training set that a single conjunctive rule is allowed to cover. With $\tau > 0$ we also allow \method\ to find a hypothesis when contradictory examples are present. However, since we are relaxing the consistency constraint, we are also increasing the likelihood of having false positives. As we will see in the experimental section, $\tau>0$ has been only used for two datasets while in all the other cases $\tau=0$ has been fixed.

\subsection{Bayes optimal approximation}
\label{find-dnf-bp}

\method\  allows us to find a DNF consistent with the training set assuming there are no instances duplicated in the positive and negative class. In other words, considering the hypothesis space of DNFs, the \method\ guarantees to find one hypothesis of the version space. This behavior resembles the Perceptron algorithm which is able to find a discriminant hyperplane whenever the data are linearly separable.
Here, we assume there exists a target MDNF that has generated data. In this case, the version space contains this hypothesis. The problem is that, when the version space is large, the chance for \method\ to find a good hypothesis decreases.  Moreover, when there are no background knowledge and the prior over the hypotheses is assumed to be uniform, it is well known that the optimal choice is to return the expected value of the decision of the hypotheses in the version space, a.k.a. Bayes Optimal Classifier (BOC):
$H_{bo}(\xx) = \text{sign}\left(\Exp_{h \in \text{V}(\cS)}[h(\xx)] \right)$.
A surrogate of the optimal classifier is the so-called Bayes Point Classifier (BPC). In this case, the center of mass $h_{cm} \in \text{V}(\cS)$ of the version space is selected and the classification is performed according to this hypothesis: $H_{bp}(\xx) = \text{sign} (h_{cm}(\xx))$.

Unfortunately, to get such hypotheses, one needs to sample uniformly from the version space and this is a very difficult task.
With the aim to approximate the BOC or the BPC, inspired by the work in \cite{DBLP:conf/nips/HerbrichG00a} applied to the perceptron algorithm, we propose to combine different DNFs obtained from multiple runs of \method\ where different orders of presentation of the examples are used.  

Specifically, \method-BO runs T instances of \method\, obtaining $T$ rule sets $\{h_1,\dots,h_T\}$. Then, an instance is classified according to the aggregated decision: $$H_{bo}(\xx) = \text{sign}\left(\sum_{t=1}^T h_t(\xx)\right).$$

We show in the experimental section that the new hypothesis significantly improves the accuracy upon individual rulesets and is competitive with state-of-the-art learning algorithms on most datasets.
Unfortunately, this aggregated hypothesis is not interpretable. For this, in the following, we propose an alternative that preserves interpretability while maintaining the accuracy of the aggregated rulesets given by \method-BO.

A ruleset decision can be represented in a vector space as follows. Consider an $R$-dimensional vector space where $R$ is the number of possible rules. Then, a ruleset can be seen as a vector $\hat{\ww} \in \Real^{R+1}$ where $\hat{w}_r = 1, r \leq R$ iff the rule indexed by $r$ is present in the ruleset, and $\hat{w}_{R+1} = -\frac{1}{2}$. An instance is represented in the same space as the vector $\hat{\xx} \in \Real^{R+1}$ where $\hat{x}_r = \llbracket \rr \succeq \xx \rrbracket, r \leq R$, and $\hat{x}_{R+1} = 1$. It can be easily verified that the ruleset decision can be given as $h(\xx) = \text{sign}(\langle\hat{ \ww} ,\hat{\xx} \rangle)$.

\method-BP is obtained by taking the approximate Bayes point obtained by averaging the hypotheses $\hat{\ww}^{(t)}$ found by running \method\ for $T$ times. Summarizing, we have 
\begin{align*}
    H_{bp}(\xx) &= \text{sign}\left(\frac{1}{T}\sum_{t=1}^T (\langle \hat{\ww}^{(t)}_{1:R}, \hat{\xx}_{1:R} \rangle - \frac{1}{2})\right) > 0\\
    &\Leftrightarrow \sum_{t=1}^T \langle \hat{\ww}^{(t)}_{1:R}, \hat{\xx}_{1:R} \rangle > \sum_{t=1}^T \frac{1}{2} = \frac{T}{2}.
\end{align*}

Noticing that $\sum_t \langle \hat{\ww}^{(t)}_{1:R}, \hat{\xx}_{1:R} \rangle$ equals the number of rules in the T rulesets $\xx$ satisfies in total, then $H_{bp}$ classifies a new instance as positive whenever the number of rules the instance satisfies in total exceeds $T/2$. Moreover, let $G$ be the set of unique discovered rules, then the decision can be compacted as $\sum_{\rr \in G} \alpha_{\rr} \llbracket \rr \succeq \xx \rrbracket > T/2$ where $\alpha_{\rr}$ is the number of times $\rr$ has been discovered.

Interestingly, this approach gives a principled method to order the discovered rules according to their weights $\alpha$. This weight represents the importance of a given rule in the decision. Hence, it makes sense to prune the set of rules by dropping the ones that have lower weights. In the experimental section we will see that doing pruning is effective both for interpretability purposes (it reduces the number of rules of explanation) and for reducing the overfit that can result when too many rules are used. Note that, when this pruning is performed, let say by retaining the first $K$ rules $G_K$, then the discriminant function needs to be changed accordingly. Namely, $H_K(\xx) = +1$ iif 
$$\sum_{\rr \in G_K} \alpha_\rr \llbracket \rr \succeq \xx \rrbracket >  \frac{\gamma_K T}{2}, \mbox{ where } \gamma_K = \frac{\sum_{\rr \in G_K} \alpha_\rr}{\sum_{\rr \in G} \alpha_\rr}.$$ 



\section{Experiments}

This section presents the experiments performed with \method. 
%
We evaluated our model on different benchmark datasets and compared it with many state-of-the-art baselines. The experiments have been executed on an Apple MacBook Pro (2019) with a 2,6 GHz 6-Core Intel Core i7. The model has been implemented in Python using  \texttt{scikit-learn}\footnote{\url{https://scikit-learn.org/stable/}}. The source code of our experiments is available at this link\footnote{\url{https://tinyurl.com/bdcwh35d}}.

To evaluate the performance of the methods, we used both \texttt{accuracy} and \texttt{F1-score}. 


\subsection{Datasets} 

We selected 11 real-world datasets, taken from the UCI Machine Learning Repository\footnote{\url{https://archive.ics.uci.edu/ml/index.php}}. The datasets are summarized in Table~\ref{tab:datasets}.

\begin{table}[h]
\small
\begin{tabular}{ l c c c c }
 \hline
 \textbf{Dataset} & \textbf{\#instances} & \textbf{\#classes} & \textbf{\#feat.} & \textbf{$\oplus$ class}  \\
 \hline
 \texttt{banknote} & 1372 & 2 & 14 & 1 \\ 
 \texttt{car} & 1728 & 6 & 4 & \textit{unacc} \\ 
 \texttt{connect-4} & 67557 & 3 & 42 & \textit{win} \\ 
 \texttt{kr-vs-kp} & 3196 & 2 & 36 & \textit{won} \\ 
 \texttt{monk-1} & 432 & 2 & 8 & 1 \\ 
 \texttt{monk-2} & 432 & 2 & 8 & 1 \\ 
 \texttt{monk-3} & 432 & 2 & 8 & 1 \\ 
 \texttt{mushroom} & 8124 & 2 & 22 & \textit{e (edible)} \\ 
 \texttt{tic-tac-toe} & 958 & 2 & 9 & \textit{positive} \\ 
 \texttt{vote} & 435 & 2 & 16 & \textit{republican} \\ 
 \texttt{wine} & 178 & 3 & 13 & 2 \\ 
 \hline
\end{tabular}
\caption{\label{tab:datasets}Real-world benchmark datasets for the rule learning task. All the datasets are discrete, except for \texttt{banknote} and \texttt{wine}.}
\end{table}

All the datasets are discrete, except for \texttt{banknote} and \texttt{wine}. Since we work with binary classification, multi-class datasets have been converted into binary classification datasets by selecting the most frequent class as the positive class. We used this procedure, replicating DRNC \cite{beck2021empirical} for all the datasets expect for \texttt{monk-1}, \texttt{monk-2}, \texttt{monk-3}, and \texttt{banknote}, where we used the value \texttt{1} as the positive class. For continuous data, we discretized the features using 10 bins, applying the \texttt{KBinsDiscretizer} provided by \texttt{scikit-learn}.

\subsection{Baselines} 

We compare \method\ with different baselines. They can be subdivided into interpretable and black-box approaches. In general, black-box models are more accurate but lack explainability, while interpretable models can generate logical rules that can be used both for prediction and explanation. 

\noindent Interpretable models: Decision Trees (\texttt{CART}), 
Logistic Regression (\texttt{LR}),
\texttt{RIPPER}\footnote{ \url{https://github.com/imoscovitz/wittgenstein}}~\cite{ripper}, 
and
Scalable Bayesian Rule List (\texttt{SBRL}\footnote{\url{https://github.com/oracle/Skater}}) \cite{sbrl}.


\noindent Black-box models: Support Vector Machines (\texttt{SVM}), and Random Forests (\texttt{RF}).



To select the hyper-parameters for the baselines, we followed the choices presented in\cite{rrl}, summarized in Table~\ref{tab:hyper-params}.

\begin{table}
\small
\centering
\begin{tabular}{ l l }
 \hline
 \textbf{Baseline} & \textbf{Hyper-parameters}  \\
 \hline
 \texttt{CART} & max\_depth $\in \{\text{None}, 5, 10, 20\}$ \\ 
 \texttt{LR} & penalty $\in \{l_1, l_2\}, C \in \{1, 4, 16, 32\},$ \\
 & solver=liblinear\\
 \texttt{RF} & n\_estimators $\in \{10, 100, 500\}$ \\ 
 \texttt{SVM} & kernel $\in \{\text{linear}, \text{rbf}, \text{poly}\},$\\
 & $C \in \{1, 4, 16, 32\}$ \\ 
 \texttt{RIPPER} & $k \in \{1, 2\}$, prune\_size $\in \{0.5, 0.33, 0.2\},$ \\
 & dl\_allowance $\in \{64, 32\}$ \\ 
 \texttt{SBRL} & $\lambda = 5, \eta = 1$, iters$=5000$, n\_chains$=20$, \\ 
 & max\_rule\_size$=3$, min\_rule\_size$=1$ \\ 
 \hline
\end{tabular}
\caption{\label{tab:hyper-params}Hyper-parameters considered in our experiments. 
For some hyper-parameters, a range of values is presented. These values are used for performing a grid search, as explained in Section~\ref{sec:exp-setting}.}
\end{table}

\subsection{Experimental setting}\label{sec:exp-setting}

To train, validate, and test our baselines, we used the following procedure:
$(1)$ the dataset is randomly split into training and test set, using a 50/50 proportion;
$(2)$ a grid search is performed on the training set using 5-fold cross-validation, with hyper-parameters reported in Table~\ref{tab:hyper-params};
$(3)$ the best model found is trained on the entire training set;
$(4)$ the metrics are computed on the test set.
The same procedure has been applied to our approach, with the difference that the second and third steps were not performed since \method\ have no hyper-parameters.
The number of iterations $T$ of \method-BO/BP has been set to 100 (\method-BO/BP$_{100}$). For \texttt{monk-2}, \texttt{monk-3} and \texttt{vote}, we set $\tau=1$, in all other cases $\tau=0$.
We repeated this process for ten runs and the test metrics have been averaged among these runs. For \texttt{connect-4}, we performed this procedure once since the dataset is computationally demanding. 
Additionally, we tried AV and OH encoding representations of the features for \texttt{CART}, \texttt{RIPPER}, and \method\, and reported the results obtained with the best-performing one. For the other baselines, we used the OH encoding representation only as they cannot use AV encoding.

\subsection{Results}\label{sec:results}

A comparison between interpretable models is presented in Table~\ref{tab:f1_int}, while a comparison between black-box models is presented in Table~\ref{tab:f1_unint}. 

\begin{table*}[ht]
\centering
\small
\begin{tabular}{ l c c c c c c  }
\hline
& \method & \method-BP$_{100}$ & RIPPER & CART & LR & SBRL\\ 
 \hline
\texttt{banknote} & 0.949 ± 0.00 & 0.954 ± 0.00 & 0.696 ± 0.02 & 0.967 ± 0.01 & \textbf{0.996} ± 0.00 & N/A \\ 
\texttt{car} & \textbf{0.990} ± 0.00 & 0.989 ± 0.00 & 0.988 ± 0.01 & 0.983 ± 0.00 & 0.963 ± 0.00 & 0.966 ± 0.00 \\ 
\texttt{kr-vs-kp} & 0.987 ± 0.00 & \textbf{0.993} ± 0.00 & 0.981 ± 0.00 & 0.989 ± 0.01 & 0.972 ± 0.00 & N/A \\ 
\texttt{monk-1} & \textbf{1.000} ± 0.00 & \textbf{1.000} ± 0.00 & 0.939 ± 0.06 & 0.909 ± 0.05 & 0.669 ± 0.02 & \textbf{1.000} ± 0.00 \\ 
\texttt{monk-2} & 0.768 ± 0.00 & \textbf{0.811} ± 0.02 & 0.175 ± 0.12 & 0.811 ± 0.08 & 0.038 ± 0.05 & 0.297 ± 0.27 \\ 
\texttt{monk-3} & 0.970 ± 0.00 & \textbf{0.988} ± 0.01 & 0.935 ± 0.01 & 0.986 ± 0.01 & 0.972 ± 0.01 & 0.988 ± 0.00 \\ 
\texttt{mushrooms} & \textbf{1.000} ± 0.00 & \textbf{1.000} ± 0.00 & \textbf{1.000} ± 0.00 & 1.000 ± 0.00 & \textbf{1.000} ± 0.00 & N/A \\ 
\texttt{ttt} & \textbf{1.000} ± 0.00 & \textbf{1.000} ± 0.00 & 0.986 ± 0.02 & 0.932 ± 0.02 & 0.986 ± 0.00 & 0.961 ± 0.03 \\ 
\texttt{vote} & 0.910 ± 0.00 & 0.933 ± 0.01 & 0.815 ± 0.03 & 0.933 ± 0.02 & \textbf{0.942} ± 0.01 & N/A \\ 
\texttt{wine} & 0.811 ± 0.00 & 0.878 ± 0.03 & 0.826 ± 0.09 & 0.855 ± 0.04 & \textbf{0.919} ± 0.04 & 0.719 ± 0.11 \\ 
\texttt{connect-4}$^*$ & 0.850 & \textbf{0.896} & 0.730 & 0.849 & 0.849 & N/A \\ 
\hline
\textbf{AvgRank}  & 3.00 & \textbf{1.82} & 4.32 & 3.45 & 3.50 & 4.91\\
\hline
\end{tabular}
\caption{\label{tab:f1_int} \texttt{F1-score} for interpretable models. The metric is computed on the test set and averaged across ten runs of the evaluation procedure explained in Section~\ref{sec:exp-setting}. (*) For \texttt{connect-4} the hyper-parameter $T$ of \method-BP is set to 20, and the standard deviation is not reported since the evaluation procedure has been computed once. Best results are shown in bold.}  
\end{table*}

\begin{table*}[th]
\small
\centering
    \begin{tabular}{ l c c c c c c  }
  & \method & \method-BP$_{100}$ & RIPPER & CART & LR & SBRL\\ 
 \hline
\texttt{banknote} & 0.956 ± 0.00 & 0.960 ± 0.00 & 0.786 ± 0.02 & 0.971 ± 0.01 & \textbf{0.997} ± 0.00 & N/A \\ 
\texttt{car} & \textbf{0.986} ± 0.00 & 0.984 ± 0.00 & 0.983 ± 0.01 & 0.976 ± 0.01 & 0.948 ± 0.01 & 0.953 ± 0.01 \\ 
\texttt{kr-vs-kp} & 0.986 ± 0.00 & \textbf{0.992} ± 0.00 & 0.980 ± 0.00 & 0.989 ± 0.01 & 0.971 ± 0.00 & N/A \\ 
\texttt{monks-1} & \textbf{1.000} ± 0.00 & \textbf{1.000} ± 0.00 & 0.943 ± 0.04 & 0.908 ± 0.05 & 0.746 ± 0.01 & \textbf{1.000} ± 0.00 \\ 
\texttt{monks-2} & 0.836 ± 0.00 & \textbf{0.868} ± 0.01 & 0.674 ± 0.03 & 0.868 ± 0.06 & 0.612 ± 0.03 & 0.661 ± 0.03 \\ 
\texttt{monks-3} & 0.968 ± 0.00 & \textbf{0.988} ± 0.01 & 0.935 ± 0.01 & 0.986 ± 0.01 & 0.971 ± 0.01 & 0.987 ± 0.00 \\ 
\texttt{mushrooms} & \textbf{1.000} ± 0.00 & \textbf{1.000} ± 0.00 & \textbf{1.000} ± 0.00 & 1.000 ± 0.00 & \textbf{1.000} ± 0.00 & N/A \\ 
\texttt{ttt} & \textbf{1.000} ± 0.00 & \textbf{1.000} ± 0.00 & 0.982 ± 0.03 & 0.912 ± 0.02 & 0.982 ± 0.00 & 0.949 ± 0.04 \\ 
\texttt{vote} & 0.932 ± 0.00 & 0.950 ± 0.00 & 0.868 ± 0.02 & 0.947 ± 0.02 & \textbf{0.955} ± 0.01 & N/A \\ 
\texttt{wine} & 0.864 ± 0.00 & 0.909 ± 0.02 & 0.876 ± 0.05 & 0.891 ± 0.03 & \textbf{0.939} ± 0.03 & 0.825 ± 0.06 \\ 
\hline 
\texttt{connect-4} & 0.803 & \textbf{0.860} & 0.712 & 0.801 & 0.793 & N/A \\ 
\hline
\textbf{AvgRank}  & 3.00 & \textbf{1.73} & 4.18 & 3.45 & 3.64 & 5.00
    \end{tabular}
    \caption{\label{tab:acc_int} \texttt{Accuracy} for interpretable models. The metric is computed on the test set and averaged across ten runs of the evaluation procedure explained in Section~\ref{sec:exp-setting}. (*) For \texttt{connect-4} the hyper-parameter $T$ of \method-BP is set to 20, and the standard deviation is not reported since the evaluation procedure has been computed once. Best results are shown in bold.}
    \end{table*}

Both tables present the results using the \texttt{F1-score}.
In the last line of each table we show the average rank of the compared methods. The best results are reported in bold. The results of \texttt{SBRL} are incomplete since the execution of that baseline has been interrupted due to memory errors.

By looking at Table~\ref{tab:f1_int}, we can observe that \method\ outperforms the majority of selected interpretable baselines. Interestingly, \texttt{LR}  outperforms \method\ on three datasets. For continuous datasets, namely \texttt{wine} and \texttt{banknote}, this behavior could be explained by the fact that the datasets have been discretized using ten bins, which could be suboptimal for \method.

Using \method-BP$_{100}$, as expected, the results can be further improved. In addition, it is essential to highlight that \method\ outperforms the selected baselines even on the most challenging dataset, namely \texttt{connect-4}. Overall, \method-BP$_{100}$ is the best performing approach, with an average rank of 1.82.

Table~\ref{tab:f1_unint} reports the results of black-box methods. \method-BO$_{100}$ outperforms the selected baselines on five datasets. In three of these, it has the same performance of \texttt{SVM}. Again, \texttt{wine} has shown to be a challenging dataset, probably due to its continuous nature. Both \method-BO$_{100}$ and \texttt{SVM} perform poorly compared to Random Forests. Again, this could be due to the number of bins chosen for the discretization of the continuous variables. Overall, black-box \texttt{SVM} is the best performing approach. 

    

The results for the same experiments using the \texttt{accuracy} are presented in Table~\ref{tab:acc_int} and Table~\ref{tab:acc_unint}.

\begin{table}[h]
\small
\begin{tabular}{ l c c c  }
\hline
& \method-BO$_{100}$ & SVM & RF\\ 
 \hline
\texttt{banknote} & 0.953 ± 0.00 & \textbf{0.997} ± 0.00 & 0.981 ± 0.00 \\ 
\texttt{car} & 0.988 ± 0.00 & \textbf{0.996} ± 0.00 & 0.983 ± 0.01 \\ 
\texttt{kr-vs-kp} & \textbf{0.993} ± 0.00 & 0.991 ± 0.00 & 0.988 ± 0.00 \\ 
\texttt{monk-1} & \textbf{1.000} ± 0.00 & \textbf{1.000} ± 0.00 & 0.989 ± 0.01 \\ 
\texttt{monk-2} & 0.829 ± 0.02 & \textbf{0.966} ± 0.02 & 0.537 ± 0.07 \\ 
\texttt{monk-3} & \textbf{0.988} ± 0.01 & 0.985 ± 0.01 & 0.971 ± 0.01 \\ 
\texttt{mushrooms} & \textbf{1.000} ± 0.00 & \textbf{1.000} ± 0.00 & \textbf{1.000} ± 0.00 \\ 
\texttt{ttt} & \textbf{1.000} ± 0.00 & 0.986 ± 0.00 & 0.967 ± 0.01 \\ 
\texttt{vote} & 0.933 ± 0.01 & \textbf{0.939} ± 0.02 & 0.935 ± 0.03 \\ 
\texttt{wine} & 0.878 ± 0.03 & 0.925 ± 0.02 & \textbf{0.941} ± 0.05 \\ 
\texttt{connect-4}$^*$ & 0.894 & \textbf{0.917} & 0.901 \\ 
\hline
\textbf{AvgRank}  & 2.05 & \textbf{1.50} & 2.45\\
\hline
\end{tabular}
\caption{\label{tab:f1_unint} \texttt{F1-score} for black-box models. The metric is computed on the test set and averaged across ten runs of the evaluation procedure explained in Section~\ref{sec:exp-setting}. (*) For \texttt{connect-4} the hyper-parameter $T$ of \method-BO is set to 20, and the standard deviation is not reported since the evaluation procedure has been computed once. Best results are shown in bold. }
\end{table}

\begin{table}[h]
\small
    \begin{tabular}{ l c c c  }
  & \method-BO$_{100}$ & SVM & RF\\ 
 \hline
\texttt{banknote} & 0.960 ± 0.00 & \textbf{0.997} ± 0.00 & 0.983 ± 0.00 \\ 
\texttt{car} & 0.984 ± 0.00 & \textbf{0.995} ± 0.00 & 0.976 ± 0.01 \\ 
\texttt{kr-vs-kp} & \textbf{0.992} ± 0.00 & 0.990 ± 0.00 & 0.987 ± 0.00 \\ 
\texttt{monks-1} & \textbf{1.000} ± 0.00 & \textbf{1.000} ± 0.00 & 0.989 ± 0.01 \\ 
\texttt{monks-2} & 0.884 ± 0.01 & \textbf{0.977} ± 0.01 & 0.740 ± 0.03 \\ 
\texttt{monks-3} & \textbf{0.987} ± 0.01 & 0.984 ± 0.01 & 0.970 ± 0.01 \\ 
\texttt{mushrooms} & \textbf{1.000} ± 0.00 & \textbf{1.000} ± 0.00 & \textbf{1.000} ± 0.00 \\ 
\texttt{ttt} & \textbf{1.000} ± 0.00 & 0.981 ± 0.00 & 0.956 ± 0.01 \\ 
\texttt{vote} & 0.947 ± 0.00 & \textbf{0.953} ± 0.01 & 0.950 ± 0.02 \\ 
\texttt{wine} & 0.904 ± 0.02 & 0.944 ± 0.02 & \textbf{0.954} ± 0.04 \\ 
\hline 
\texttt{connect-4} & 0.860 & \textbf{0.890} & 0.865 \\ 
\hline
\textbf{AvgRank}  & 2.05 & \textbf{1.50} & 2.45
    \end{tabular}
    \caption{\label{tab:acc_unint} \texttt{Accuracy} for black-box models. The metric is computed on the test set and averaged across ten runs of the evaluation procedure explained in Section~\ref{sec:exp-setting}. (*) For \texttt{connect-4} the hyper-parameter $T$ of \method-BO is set to 20, and the standard deviation is not reported since the evaluation procedure has been computed once. Best results are shown in bold. }
    \end{table}


\subsection{The effect of pruning in \method}
Here we show that pruning (Sec.~\ref{sec:pruning}) may help both the interpretability as well as the accuracy of the model. We consider the \method's hypothesis before and after pruning on \texttt{tic-tac-toe} using 50\% of the dataset for training with AV encoding. 
Before pruning, the hypothesis consists of 9 conjunctive rules: 8 describe the winning three in a row for $\times$, and one rule describes a naught, a cross, and a blank in the middle row. 
Clearly, this last rule is incorrect allowing to get false positives (confirmed by the 99.5\% of accuracy on the test set). The pruning phase correctly removes such a rule, leaving the hypothesis with only the correct 8 winning combinations. It could also be observed that, in the bucket of examples associated with the wrong rule, that was the sixth to be created, there were only 6 examples where 5 of them were covered by the antidiagonal rule and one by the bottom row rule, both of them created afterwards.

\subsection{Pruning the \method-BP rules}
In Section \ref{find-dnf-bp} we have discussed how the rules discovered by \method-BP can be pruned. The plot in Figure \ref{fig:pruning} shows how the training and testing accuracies vary with respect to the degree of pruning performed in 4 different datasets. 

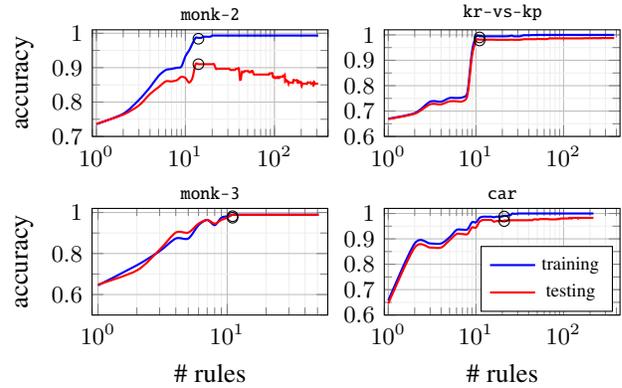
\begin{figure}[h]
\label{fig:pruning}
\begin{tikzpicture}
\begin{axis}[
    xmin = 0.9, xmax = 400,
    ymin = 0.7, ymax = 1.01,
    xmode=log,
    grid = both,
    minor tick num = 1,
    major grid style = {lightgray},
    minor grid style = {lightgray!25},
    width = 0.55\columnwidth,
    height = 0.35\columnwidth,
    ylabel = {accuracy},
    ylabel near ticks,
    every tick label/.append style={font=\small},
    legend style={font=\scriptsize},
    legend pos = south east,
    title={\scriptsize{\texttt{monk-2}}},
    title style={yshift=-6pt}]

\addplot[
    smooth,
    blue,
    thick
] table[x expr=\coordindex,y=y] {monks2_train.txt}; 

\addplot[
    smooth,
    red,
    thick
] table[x expr=\coordindex,y=y] {monks2_test.txt};

\addplot[solid, mark=o] coordinates { (14,0.984)  };
\addplot[solid, mark=o] coordinates { (14,0.91)  };

\end{axis}
\end{tikzpicture}
\begin{tikzpicture}
\begin{axis}[
    xmin = 0.9, xmax = 440,
    ymin = 0.6, ymax = 1.02,
    xmode=log,
    grid = both,
    minor tick num = 1,
    major grid style = {lightgray},
    minor grid style = {lightgray!25},
    width = 0.55\columnwidth,
    height = 0.35\columnwidth,
    every tick label/.append style={font=\small},
    legend style={font=\scriptsize},
    legend pos = south east,
    title={\scriptsize{\texttt{kr-vs-kp}}},
    title style={yshift=-6pt}]

\addplot[
    smooth,
    blue,
    thick
] table[x expr=\coordindex,y=y] {krvskp_train.txt}; 

\addplot[
    smooth,
    red,
    thick
] table[x expr=\coordindex,y=y] {krvskp_test.txt};

\addplot[solid, mark=o,] coordinates { (11,0.99)  };
\addplot[solid, mark=o,] coordinates { (11,0.978)  };

\end{axis}
\end{tikzpicture}

\begin{tikzpicture}
\begin{axis}[
    xmin = 0.9, xmax = 60,
    ymin = 0.5, ymax = 1.02,
    xmode=log,
    grid = both,
    minor tick num = 1,
    major grid style = {lightgray},
    minor grid style = {lightgray!25},
    width = 0.55\columnwidth,
    height = 0.35\columnwidth,
    xlabel = {\# rules},
    ylabel = {accuracy},
    ylabel near ticks,
    every tick label/.append style={font=\small},
    legend style={font=\scriptsize},
    legend pos = south east,
    title={\scriptsize{\texttt{monk-3}}},
    title style={yshift=-6pt}]

\addplot[
    smooth,
    blue,
    thick
] table[x expr=\coordindex,y=y] {monks3_train.txt}; 

\addplot[
    smooth,
    red,
    thick
] table[x expr=\coordindex,y=y] {monks3_test.txt};

\addplot[solid, mark=o] coordinates { (11,0.98188)  };
\addplot[solid, mark=o] coordinates { (11,0.9746)  };

\end{axis}
\end{tikzpicture}
\begin{tikzpicture}
\begin{axis}[
    xmin = 0.9, xmax = 440,
    ymin = 0.6, ymax = 1.02,
    xmode=log,
    grid = both,
    minor tick num = 1,
    major grid style = {lightgray},
    minor grid style = {lightgray!25},
    width = 0.55\columnwidth,
    height = 0.35\columnwidth,
    xlabel = {\# rules},
    every tick label/.append style={font=\small},
    legend style={font=\scriptsize},
    legend pos = south east,
    title={\scriptsize{\texttt{car}}},
    title style={yshift=-6pt}]

\addplot[
    smooth,
    blue,
    thick
] table[x expr=\coordindex,y=y] {car_train.txt}; 

\addplot[
    smooth,
    red,
    thick
] table[x expr=\coordindex,y=y] {car_test.txt};

\addplot[solid, mark=o] coordinates { (21,0.988)  };
\addplot[solid, mark=o] coordinates { (21,0.97)  };

\legend{training, testing}
\end{axis}
\end{tikzpicture}
\caption{The plots (log scale) show the behaviour of the training and test accuracies varying the number of kept rules in the \method-BP pruning phase. Rules are sorted in decreasing order of importance. It is possible to notice that there are cases of overfitting where removing the least important rules improve the test accuracy. In any case, more than 90\% are removed thus increasing the interpretability of the method with no loss in accuracy.
The black circle highlights the accuracies corresponding to the 99\% threshold.
\label{fig:pruning}}
\end{figure}

Interestingly, we can see that, when the training accuracy stabilizes, this also corresponds to the best value for the test accuracy. We can also notice that, in one case, \texttt{monk-2}, this point corresponds to a drop of the test accuracy (overfitting). This behavior suggests a strategy for the selection of the optimal number $K$ of rules to keep. Specifically, we can set a threshold and select the minimum number of rules such that the training accuracy is above that threshold. In the figure, it is shown that, setting a threshold of $0.99 \times \text{tr\_acc}$, where $\text{tr\_acc}$ is the training accuracy of the complete rule set without pruning,  can  drastically reduce the number of rules, thus improving the interpretability of the ruleset with no reduction in accuracy. It can also be beneficial with respect to the test accuracy when overfitting occurs (see for example Figure~\ref{fig:pruning}). 

\section{Conclusions}

This paper has proposed a novel methodology that tries to approximate the Bayes optimal classifier in the hypothesis space of MDNF or rule sets.
The method has demonstrated to discover rules that are very accurate and still interpretable.
We have also provided a methodology to reduce the number of rules of the final hypothesis, thus improving interpretability with no loss in classification performance. We are currently working on a principled extension of this methodology to multiclass classification. In the future, we aim to adjust the method to make it able to handle continuous variables.

\bibliographystyle{named}
\bibliography{ijcai22}


\end{document}